\documentclass[conference,10pt]{IEEEtran}
\IEEEoverridecommandlockouts
\usepackage[T1]{fontenc}
\usepackage{type1cm}
\usepackage{url}
\usepackage{fix-cm}
\usepackage{graphicx,color,comment}
\usepackage{algpseudocode}
\usepackage[table]{xcolor}
\usepackage{amsmath, amsthm, amsfonts, amssymb, amsbsy,nccmath}
\usepackage{algorithm}
\usepackage{enumerate}
\usepackage{lipsum}

\usepackage{textgreek}
\usepackage{comment}
\usepackage{textcomp}
\usepackage{tikz}
\usetikzlibrary{positioning, fit, backgrounds, arrows.meta, calc, shapes.geometric, decorations.pathreplacing}
\usepackage{booktabs}
\usetikzlibrary{positioning, arrows.meta, shapes.geometric, calc, decorations.pathreplacing}

\usepackage[sort,compress]{cite}
\usepackage{epsfig}
\usepackage{epstopdf}
\usepackage{mathtools}
\usepackage{dsfont}
\usepackage{cleveref}
\usepackage{subfig}
\usepackage{stfloats}
\usepackage{fontawesome5}

\theoremstyle{definition}
\newtheorem{definition}{Definition}[section]

\usepackage[inline]{enumitem}   
\makeatletter
\newcommand{\inlineitem}[1][]{%
\ifnum\enit@type=\tw@
    {\descriptionlabel{#1}}
  \hspace{\labelsep}%
\else
  \ifnum\enit@type=\z@
       \refstepcounter{\@listctr}\fi
    \quad\@itemlabel\hspace{\labelsep}%
\fi}
\makeatother

\newtheorem{theorem}{Theorem}

\usepackage{float}
\floatstyle{plaintop}
\restylefloat{table}
\newtheorem{corollary}{Corollary}[theorem]

\makeatletter
\newcommand\fs@spaceruled{\def\@fs@cfont{\bfseries}\let\@fs@capt\floatc@ruled
  \def\@fs@pre{\vspace{0.5\baselineskip}\hrule height.8pt depth0pt \kern2pt}%
  \def\@fs@post{\kern1pt\hrule\relax}%
  \def\@fs@mid{\kern2pt\hrule\kern2pt}%
  \let\@fs@iftopcapt\iftrue}
  \newcommand\fs@betterruled{%
  \def\@fs@cfont{\bfseries}\let\@fs@capt\floatc@ruled
  \def\@fs@pre{\vspace*{5pt}\hrule height.8pt depth0pt \kern2pt}%
  \def\@fs@post{\kern2pt\hrule\relax}%
  \def\@fs@mid{\kern2pt\hrule\kern2pt}%
  \let\@fs@iftopcapt\iftrue}
\floatstyle{betterruled}
\restylefloat{algorithm}
\makeatother

\usepackage{footnote}

\begin{document}

\title{Not All Symbols Are Equal: Importance-Aware Constellation Design for Semantic Communication\vspace{-4mm}}
\vspace{-6mm}
\author{ Albert Shaju\IEEEauthorrefmark{1}, Christo Kurisummoottil Thomas\IEEEauthorrefmark{1},  and Mayukh Roy Chowdhury\IEEEauthorrefmark{2}\\  \IEEEauthorrefmark{1} Department of Electrical and Computer Engineering, Worcester Polytechnic Institute, Worcester, MA, USA.\\  
\IEEEauthorrefmark{2} Nokia Bell Labs, Bengaluru, India.\\
Emails:ashaju@wpi.edu,cthomas2@wpi.edu,mayukh.roy\_chowdhury@nokia-bell-labs.com\ \vspace{-0.5cm}\thanks{This research was supported by AI-RAN Alliance innovation fund. }}
\maketitle
\vspace{-4mm}
\begin{abstract}
Semantic communication systems for goal-oriented
transmission must protect task-relevant information
not only through source compression but also via
 physical layer mapping. Existing approaches decouple constellation design and semantic encoding, exposing critical symbols to channel errors at the same rate as irrelevant ones. Contrary to this, in this paper, a joint
semantic-physical layer framework is proposed, which is composed of a
vector quantized-variational autoencoder that
extracts discrete latent concepts, a semantic
criticality indicator (SCI) that scores each concept
by task relevance, and a deep reinforcement learning
agent that dynamically selects the transmission
subset based on instantaneous channel conditions.
At the physical layer, a learned semantic-aware
$M$-QAM constellation assigns symbol positions
according to joint co-occurrence statistics and SCI
scores, departing from the uniform spacing and Gray
coding of standard $M$-QAM which minimizes average
BER without regard for semantic content. We introduce a novel semantic symbol vulnerability
(SSV) metric and a semantic protection probability (SPP) to quantify the exposure of
task-critical symbols to decoding errors, and
prove that any Gray-coded constellation is
strictly suboptimal in SCI-Weighted SSV whenever
the source exhibits non-uniform semantic importance
and co-occurrence statistics. Simulation results
demonstrate that the proposed constellation achieves
near $100\%$ SPP
across modulation orders from 4-QAM to 1024-QAM
versus $50\%$ for standard constellations at high
spectral efficiency, a $21$:$1$ compression ratio with semantic quality
above $0.9$, generalizing across MNIST,
Fashion-MNIST, and FSDD without modification.
\end{abstract}

\vspace{-2mm}\section{Introduction}
\vspace{-1mm}
The design of wireless communication systems has historically optimized bit-level fidelity, while being agnostic to the meaning carried by transmitted bits. While this is theoretically justified under Shannon's framework for general sources \cite{shannon1948mathematical}, it becomes bandwidth inefficient when the RX only intends to perform inference on the received messages. \emph{Semantic communication (SC)} \cite{kountouris2021semantics} and \cite{chaccour2024less} addresses this  by jointly optimizing source representation and transmission schemes to maximize downstream task performance rather than reconstruction fidelity.
Recent advances in deep learning have enabled end-to-end learned SC systems that directly map source data to transmitted signals. Vector quantized-variational autoencoders (VQ-VAE) \cite{vandenoord2017neural} produce compact discrete representations suitable for digital transmission, and reinforcement learning (RL) agents have demonstrated the ability to adaptively rate-control semantic payloads in response to channel state. However, these application layer level advances have not been matched by equivalent progress at the physical layer (PHY) modulation stage. Once semantic concepts are compressed and selected, they are typically either transmitted as continuous analog signals over learned channel mappings, foregoing compatibility with standard digital infrastructure, or mapped to Gray-coded constellations that optimize average bit error rate (BER) uniformly with no consideration for which symbols carry semantically critical information.

Parallel to SC, the authors of \cite{OShea2017DeepLearningPHY} and \cite{Stark2019JointConstellationShaping} pioneered the view of
a communication system as an autoencoder, jointly optimizing
transmitter (TX) and receiver (RX) for minimum block error rate and
demonstrating that learned constellations can outperform
standard QAM for specific channels.   However, all prior
learned constellation work \cite{OShea2017DeepLearningPHY,Stark2019JointConstellationShaping,Tung2022DeepJSCCQ} optimizes for average BER
uniformly across all symbols, with no mechanism to
differentiate protection or adapt information rates based on the semantic importance
of individual concept indices. Importance-aware transmission has been studied
at the resource allocation level, where deep reinforcement learning (DRL)-based
schemes adaptively assign bandwidth and
quantization bits based on semantic
relevance~\cite{Wang2024SemanticResourceAllocation} and \cite{Park2023SemanticProtocols6G},
and task-oriented rate control has been explored
using information bottleneck
principles~\cite{Shao2021TaskOrientedCommunication}.  Unequal error
protection for semantic features has been proposed via
proactive importance-ordered
restructuring~\cite{Zhan2026RobustSemanticCommunications}, prioritizing transmission
of critical features.  In contrast to these approaches, which operate
at the source or scheduling layer and leave the
physical constellation unchanged, our work is
the first to embed semantic importance directly
into the constellation assignment itself,
closing the gap between semantic-aware source
coding and PHY modulation.


The main contribution of this paper is a joint
semantic PHY framework that co-designs the
constellation assignment with the semantic importance
and statistical co-occurrence structure of the learned
concept vocabulary.  
First, we propose a \emph{learned
semantic-aware constellation mapper} whose complex
symbol coordinates are continuous trainable variables
optimized by an \emph{semantic criticality indicator (SCI)}-weighted loss, establishing a
direct one-to-one mapping between VQ-VAE concept
indices and physical symbols.
Second, we introduce
the SCI-weighted semantic symbol vulnerability (SSV), 
 and semantic  protection probability (SPP)
 as
novel metrics for quantifying the exposure of
task-critical symbols to decoding errors and
the degree to which a constellation
preferentially protects semantically important
concepts.  Third, we prove that any Gray-coded
$M$-QAM constellation is strictly suboptimal in
SCI-Weighted SSV whenever the source exhibits
non-uniform semantic importance and
co-occurrence statistics, characterize
the protection gap in closed form, and establish via
corollary the BER-semantic error decoupling observed
empirically. Finally, simulation results demonstrate consistently
higher semantic quality and compression than standard
$M$-QAM across all modulation orders, near $100\%$
SPPR up to 1024-QAM versus $50\%$ for Gray-coded
baselines, generalizing across visual and acoustic
domains without modification.

\vspace{-3mm}\section{Problem Formulation and System Model}\vspace{-2mm}
\label{System_Model}
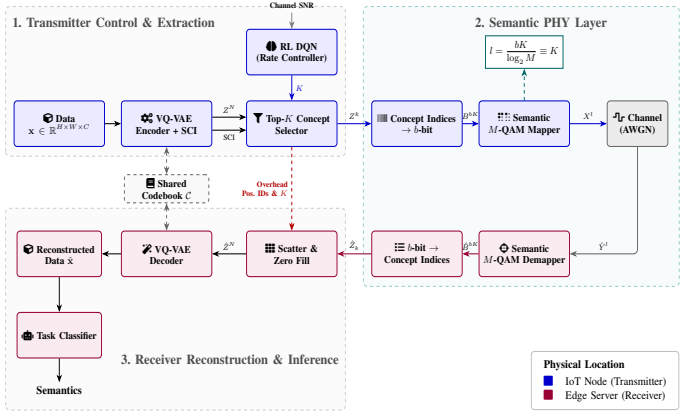
\begin{figure}[t]
    \centering
    \resizebox{0.5\textwidth}{0.23\textheight}{
\begin{tikzpicture}[>=Stealth, thick]

\tikzstyle{iot} = [rectangle, draw=blue!80!black, thick, fill=blue!8, text width=2.8cm, align=center, minimum height=1.4cm, rounded corners=3pt, font=\small\bfseries]
\tikzstyle{server} = [rectangle, draw=purple!80!black, thick, fill=purple!8, text width=2.8cm, align=center, minimum height=1.4cm, rounded corners=3pt, font=\small\bfseries]
\tikzstyle{channel} = [rectangle, draw=gray!80!black, thick, fill=gray!15, text width=1.8cm, align=center, minimum height=1.4cm, rounded corners=3pt, font=\small\bfseries]
\tikzstyle{cb} = [rectangle, draw=gray!80!black, dashed, thick, fill=gray!5, text width=2.6cm, align=center, minimum height=0.8cm, rounded corners=3pt, font=\small\bfseries]

\node[iot] (img) at (0,0) {\faCube\ Data\\$\mathbf{x} \in \mathbb{R}^{H \times W \times C}$};
\node[iot] (enc) at (3.6,0) {\faCogs\ VQ-VAE\\Encoder + SCI};
\node[iot] (sel) at (7.8,0) {\faFilter\ Top-$K$ Concept\\Selector};

\node[iot] (i2b) at (12.0,0) {\faBarcode\ Concept Indices\\$\rightarrow$ $b$-bit};
\node[iot] (map) at (15.6,0) {\faBraille\ Semantic\\$M$-QAM Mapper};
\node[channel] (chan) at (19.4,0) {\faWaveSquare\ Channel\\(AWGN)};
\node[server] (demap) at (15.6,-4) {\faCrosshairs\ Semantic\\$M$-QAM Demapper};
\node[server] (b2i) at (12.0,-4) {\faList\ $b$-bit $\rightarrow$\\Concept Indices};

\node[server] (scat) at (7.8,-4) {\faTh\ Scatter \&\\Zero Fill};
\node[server] (dec) at (3.6,-4) {\faMagic\ VQ-VAE\\Decoder};
\node[server, text width=2.6cm] (out) at (0,-4) {\faCube\ Reconstructed\\Data $\hat{\mathbf{x}}$};

\node[iot] (rl) at (7.8, 2.2) {\faBrain\ RL DQN\\(Rate Controller)};
\node[server, text width=2.6cm] (task) at (0, -6.5) {\faRobot\ Task Classifier};
\node[font=\bfseries] (sem) at (0, -8.2) {Semantics};

\node[cb] (codebook) at (3.6, -2) {\faBook\ Shared\\Codebook $\mathcal{C}$};

\node[draw=teal!80!black, thick, fill=white, inner sep=4pt, font=\small] (eq) at (15.6, 2.2) {$\displaystyle l = \frac{b K}{\log_2 M} \equiv K$};

\draw[->] (img) -- (enc);
\draw[->] ([yshift=2mm]enc.east) -- node[above, font=\scriptsize] {${Z}^N$} ([yshift=2mm]sel.west);
\draw[->] ([yshift=-2mm]enc.east) -- node[below, font=\scriptsize] {SCI} ([yshift=-2mm]sel.west);
\draw[->, blue!80!black] (sel.east) -- node[above, font=\scriptsize, text=black] {${Z}^k$} (i2b.west);

\draw[->, blue!80!black] (i2b.east) -- node[above, font=\scriptsize, text=black] {$B^{bK}$} (map.west);
\draw[->, blue!80!black] (map.east) -- node[above, font=\scriptsize, text=black] {$X^l$} (chan.west);
\draw[->, gray!80!black, rounded corners=4pt] (chan.south) |- node[pos=0.75, above, font=\scriptsize, text=black] {$\hat{Y}^l$} (demap.east);
\draw[->, purple!80!black] (demap.west) -- node[above, font=\scriptsize, text=black] {$\hat{B}^{bK}$} (b2i.east);
\draw[->, purple!80!black] (b2i.west) -- node[above, font=\scriptsize, text=black] {$\hat{{Z}}_k$} (scat.east);

\draw[->] (scat.west) -- node[above, font=\scriptsize] {$\hat{{Z}}^N$} (dec.east);
\draw[->] (dec.west) -- (out.east);
\draw[->] (out.south) -- (task.north);
\draw[->] (task.south) -- (sem.north);

\draw[<-, gray] (rl.north) -- ++(0,0.5) node[above, font=\scriptsize\bfseries, text=black] {Channel SNR};
\draw[->, blue!80!black] (rl.south) -- node[right, font=\scriptsize\bfseries] {$K$} (sel.north);

\draw[->, dashed, red!70!black] (sel.south) -- node[left, font=\scriptsize\bfseries, align=right, text=red!70!black] {Overhead\\Pos. IDs \& $K$} (scat.north);

\draw[<->, dashed, gray!80!black] (codebook.north) -- (enc.south);
\draw[<->, dashed, gray!80!black] (codebook.south) -- (dec.north);

\draw[->, dashed, teal!80!black] (map.north) -- (eq.south);

\begin{scope}[on background layer]
    \filldraw[fill=gray!3, draw=gray!40, dashed, thick, rounded corners] (-1.8, -1.0) rectangle (9.6, 3.6);
    \node[anchor=north west, font=\large\bfseries\color{black!70}] at (-1.7, 3.4) {1. Transmitter Control \& Extraction};

    \filldraw[fill=gray!3, draw=gray!40, dashed, thick, rounded corners] (-1.8, -8.8) rectangle (9.6, -2.6);
    \node[anchor=south east, font=\large\bfseries\color{black!70}] at (9.5, -7.5) {3. Receiver Reconstruction \& Inference};

    \filldraw[fill=teal!3, draw=teal!40, dashed, thick, rounded corners] (10.2, -5.0) rectangle (20.8, 3.6);
    \node[anchor=north east, font=\large\bfseries\color{black!70}] at (18.5, 3.4) {2. Semantic PHY Layer};
\end{scope}

\node[draw=gray!50, thick, fill=white, rounded corners=3pt, inner sep=6pt, anchor=south east] at (20.8, -8.8) {
    \begin{tabular}{ll}
        \multicolumn{2}{l}{\textbf{Physical Location}} \\[1mm]
        \textcolor{blue!80!black}{\faSquare} & IoT Node (Transmitter) \\
        \textcolor{purple!80!black}{\faSquare} & Edge Server (Receiver) \\
    \end{tabular}
};

\end{tikzpicture}
    }
   \vspace{-5mm} \caption{Overall Architecture of the proposed  SC system.}
    \label{fig:main_architecture}
\vspace{-8mm}\end{figure}
Consider a resource-constrained sensor device, such as a low-power IoT node deployed in an industrial or smart city environment, that acquires multidimensional data $\boldsymbol{x} \in \mathbb{R}^{H \times W \times C}$, where $H$, $W$, and $C$ denote the spatial and feature channel dimensions of the observation, respectively. The node transmits this data over a bandwidth-limited wireless channel to an edge server. The edge server then utilizes a hosted neural inference engine to perform a downstream inference task $\mathcal{T}$ on the reconstructed signal.
Unlike a conventional communication link where both endpoints share
reconstruction as the objective, the sensor has no inference capability
and transmits solely to enable accurate task execution at the server.
The sensor operates under a strict transmit power budget $P$ and must
minimize the number of transmitted physical symbols $\ell$ representing $\boldsymbol{x}$ to reduce
both bandwidth consumption and over the air latency. The edge server
periodically feeds back the estimated channel SNR to the sensor over
a reliable low-rate control channel, enabling adaptive payload
selection. We model this feedback as error-free and instantaneous,
consistent with standard assumptions in the adaptive modulation
literature~\cite{goldsmith1998adaptive}. The fundamental challenge
is therefore to compress the source data into the minimum physical
symbols that preserves the semantic content required for accurate
inference at the server. This should ensure that the most task-critical
features are physically protected against channel impairments at the
modulation layer. 
Crucially, our architecture maintains structural compatibility with conventional PHY setups, enabling seamless integration into existing systems. The architecture of the proposed end-to-end SC system, as shown in Fig.~\ref{fig:main_architecture}, are discussed next.
\vspace{-2mm}\subsection{SC Model}
\vspace{-0mm}\subsubsection{Semantic Extraction and Control}\vspace{-2mm}

At the TX, a VQ-VAE encoder is adopted because the variational prior regularizes the encoder output distribution ensuring stable and semantically coherent latent representations, while vector quantization produces a learned discrete codebook $\mathcal{C}$ whose entries are jointly optimized with the encoder and downstream task loss. The encoder maps the input $\boldsymbol{x}$ to $N$ continuous embeddings $\boldsymbol{z}_{i} \in \mathbb{R}^D$, defined collectively as $\boldsymbol{Z} \in \mathbb{R}^{N\times D}$. Each continuous embedding is quantized into a discrete latent vector $\boldsymbol{z}_{q,i}$ (collectively as $\boldsymbol{Z}_q$) via a nearest-neighbor lookup in $\mathcal{C}$:
$\boldsymbol{z}_{q,i} = \boldsymbol{e}_{k^*}, \quad \text{where} \quad
  k^* = \operatornamewithlimits{argmin}\limits_{k, \textrm{s.t.}\,\,e_k\in \mathcal{C} }
  \|\boldsymbol{z}_{i} - \boldsymbol{e}_k\|,
$
producing the full discrete concept vector $Z^N \in \mathbb{Z}^N$. A SCI network operates in parallel,
assigning each concept a continuous SCI score
$I_i \in (0,1)$. A DRL
agent, via a deep Q-network (DQN), observes the SNR fed back from the edge server
and selects the optimal transmission subset size $K \leq N$. Top-$K$ concepts and positional IDs are extracted to enable adaptive transmission of essential semantics.
\vspace{-0mm}\subsubsection{Semantic PHY}\vspace{-1mm}

The $K$ selected concept indices are modulated using a learned
semantic $M$-QAM constellation $\mathcal{X} = \{x_1,\ldots,x_M\}
\subset \mathbb{C}$, where $M = |\mathcal{C}|$ and $x_i$ is the constellation symbol. Setting the
codebook size equal to the modulation order establishes a
direct one-to-one mapping between each $b = \log_2 M$-bit
concept index and a distinct physical symbol. The constellation coordinates are trainable variables
subject to the average power constraint:
$
  \frac{1}{M}\sum_{i=1}^{M}|x_i|^2 \leq P
$.  The resulting $\ell = K$ symbols, defined as $X^l$
traverse a wireless channel, yielding
received symbols ${Y}^l$. 
\subsubsection{Semantic Reconstruction and Task Execution}

At the edge server, the semantic $M$-QAM demapper recovers
the $K$ concept indices $\hat{i}$ from the noisy received symbol
$\hat{{y}}$ via minimum-distance detection: $
  \hat{i} = \operatornamewithlimits{argmin}\limits_{i \in \{1,\ldots,M\}}
  \|\hat{y} - x_i\|^2.$
%
At the RX, using the shared codebook $\mathcal{C}$ and the control-path Positional IDs, the recovered vectors are scattered back to their original spatial coordinates within an $N$-slot grid. The remaining $N-K$ unselected positions are zero-filled to yield the reconstructed $\hat{\boldsymbol{{Z}}}$. The VQ-VAE decoder
 reconstructs $\hat{\boldsymbol{x}}$ from
$\hat{\boldsymbol{{Z}}}$ and a frozen task classifier $\mathcal{T}$ evaluates
the downstream inference on the reconstructed
image. 
\vspace{-1mm}\section{Proposed AI Architecture for Semantic PHY}


\vspace{-1mm}\subsection{Task-Specific Differentiable Classifier}\vspace{-1mm}

To evaluate semantic quality, we pre-train a lightweight
multi-layer perceptron (MLP) classifier $\mathcal{T}$ on
clean source images. The network maps the flattened
data ($\boldsymbol{x}$) vector to a class probability distribution
$\hat{\boldsymbol{p}} \in \mathbb{R}^{|\mathcal{Y}|}$ via two
fully connected hidden layers with ReLU activations and
dropout regularization, trained to minimize the sparse
categorical cross-entropy loss:
$
  \mathcal{L}_{\mathrm{cls}}(y, \hat{\boldsymbol{p}}) = -\log \hat{p}_y
$, where $\hat{p}_y = p(\mathcal{T}(\hat{\boldsymbol{x}}) = y)$. 

\vspace{-2mm}\subsection{SCI-Weighted VQ-VAE and SCI}\vspace{-2mm}

The core compression engine is an SCI-weighted
VQ-VAE (S-VQ-VAE) that jointly learns discrete semantic
representations and their task relevance. The SCI network is implemented as a two-layer MLP with sigmoid
output and requires no explicit importance supervision.
Its parameters are trained end-to-end via gradients
of $\mathcal{L}_{\mathrm{cls}}$ that backpropagate
through the decoder and across the quantization step
via the straight-through estimator (STE) \cite{Bengio2013StochasticNeurons}, implicitly forcing higher scores onto
concepts whose presence improves downstream
classification accuracy. 

\subsubsection{SCI-Weighted Forward Pass}

For a given $K$, each
continuous embedding $\boldsymbol{z}_{i}$ is scaled
by its normalized SCI weight via element-wise
multiplication, 
%
%
High-importance concepts (i.e., concepts with high SCI score) receive near-unit weights
and survive codebook quantization faithfully, while
low-importance concepts are attenuated toward zero
and effectively suppressed.
Here, we use a temperature parameter $\tau$ to
progressively harden the soft selection into a discrete
Top-$K$ mask at inference. During inference the soft
SCI weighting is replaced by hard
Top-$K$ masking.

\begin{table}[t]
\centering
\caption{System Hyperparameters and Simulation Setup}\vspace{-2mm}
\label{tab:hyper}
\renewcommand{\arraystretch}{0.9}
\scriptsize 
\begin{tabular}{ll|ll}
\toprule
\textbf{Parameter} & \textbf{Value} & \textbf{Parameter} & \textbf{Value} \\
\midrule
\rowcolor{gray!10} \multicolumn{4}{l}{\textit{S-VQ-VAE \& SCI Network}} \\
$D$ (Latent Dim)  & 64             & $N$ (Concept Slots) & 64             \\
Hidden Units      & [512, 256]     & Batch Size          & 128            \\
$\beta$ (Commit)  & 0.25           & $\tau$ (Temp.)      & $1.0 \to 0.1$  \\
$\lambda_{\text{sem}}$ & 1.0       & Learning Rate       & $10^{-3}$      \\
\midrule
\rowcolor{gray!10} \multicolumn{4}{l}{\textit{DQN Rate Controller}} \\
State Dim         & 2              & $|\mathcal{A}|$ (Actions) & 12             \\
Hidden Units      & [64, 64]       & $\gamma$ (Gamma)    & 0.99           \\
Buffer Size       & $10^4$         & $\epsilon$ (Epsilon) & $1.0 \to 0.01$ \\
$K_{\min}, K_{\max}$ & 5, 64        & Bonus $\alpha$      & 0.2            \\
\midrule
\rowcolor{gray!10} \multicolumn{4}{l}{\textit{Semantic PHY \& Channel}} \\
$M$-QAM Orders    & $\{4 \dots 1024\}$ & Training Steps    & 4000           \\
Init. Grid        & Rect. QAM      & SNR Range           & $[-10, 20]$\,dB \\
Optimizer         & Adam           & Learning rate (PHY)            & $10^{-3}$      \\
\bottomrule
\end{tabular}
\vspace{-6mm} 
\end{table}

\subsubsection{Two-Phase Training}

Training proceeds in two phases to ensure stable
convergence. In \emph{Phase~1}, which is the representation
learning, the model trains as a standard VQ-VAE to
establish robust reconstruction capability:
$ \mathcal{L}_{\mathrm{VQ}}
  = \mathbb{E}\!\left[\|\boldsymbol{x} - \hat{\boldsymbol{x}}
    \|_2^2\right]
  + \beta\,\|\mathrm{sg}[\boldsymbol{Z}_{q}]
    - \boldsymbol{Z}\|_2^2,$
where $\mathrm{sg}[\cdot]$ denotes the stop-gradient
operator and $\beta$ is the commitment loss
weight. The codebook entries are updated via the exponential
moving average of assigned encoder outputs, following
the standard VQ-VAE training procedure~\cite{vandenoord2017neural}.
 In \emph{Phase~2}, called semantic activation, the
downstream task loss is activated, reorganizing the
codebook geometry around task-relevant features:
\vspace{-2mm}\begin{equation}
  \mathcal{L}_{\mathrm{total}}
  = \mathcal{L}_{\mathrm{VQ}}
  + \lambda_{\mathrm{sem}}\,
    \mathcal{L}_{\mathrm{cls}}(y, \hat{\boldsymbol{x}}),
  \label{eq:total_loss}
\vspace{-1mm}\end{equation}
where $\lambda_{\mathrm{sem}}$ is the semantic
loss weight. Because the
importance weights are applied before quantization,
gradients from $\mathcal{L}_{\mathrm{cls}}$ backpropagate
through the decoder, across the quantization step via
the STE, and directly into both the SCI network and the
codebook entries, forcing the learned discrete
vocabulary to concentrate task-relevant structure
into a small subset of actively used codewords. 

\vspace{-2mm}\subsection{Deep Reinforcement Learning Rate Controller}\vspace{-1mm}

Adaptive semantic concept selection is formulated as a Markov
decision process (MDP) solved by a deep Q-network (DQN)
agent. At each transmission interval $t$, the agent
observes state $s_t = \mathrm{SNR}_{n}$,
the channel SNR normalized to $[0,1]$ from the feedback
path described in Section~\ref{System_Model}, and selects
action $a_t \equiv K \in \mathcal{A}$, where $\mathcal{A}$
is a discrete set of $|\mathcal{A}|$ uniformly spaced
values spanning $[K_{\min}, K_{\max}]$.
The reward function is a multi-objective formulation
balancing \emph{semantic quality, bandwidth efficiency, and
PHY reliability}:
\vspace{-2mm}\begin{equation}
  r_t = Q_{\mathrm{task}}
      + \mathcal{B}_{\mathrm{comp}}
      + \mathcal{B}_{a}
      - \lambda\,P_e
      - \mathcal{P}(K),
  \label{eq:reward}
\vspace{-1mm}\end{equation}
where $Q_{\mathrm{task}} \in [0,1]$ is the downstream
task accuracy,   $P_e$ the BER, and $\mathcal{P}(K)$ is a
regularization term that
penalizes selection of extreme payload sizes
$K \in \{K_{\min}, K_{\max}\}$, preventing the
agent from collapsing to a degenerate policy that
ignores channel conditions.
Two conditional bonuses guide exploration.
$\mathcal{B}_{\mathrm{comp}} = \alpha \ln(N/K)$
if $Q_{\mathrm{task}} > Q_0$ and zero otherwise,
rewarding compression only when semantic quality
is preserved, and $\mathcal{B}_{a}$ incentivizes
fewer concepts at high SNR and more at low SNR
to enforce channel-adaptive behavior.  The DQN employs
an MLP policy network with $\epsilon$-greedy exploration,
experience replay, and a periodically synchronized target
network for Bellman stability. 
\vspace{-2mm}\subsection{Learned Semantic $M$-QAM Constellation}


\vspace{-1mm}\subsubsection{SCI-Weighted Constellation Loss}

Unlike standard QAM, which minimizes average BER without regard to semantics, the proposed mapper optimizes an SCI-weighted loss that penalizes errors on critical symbols more severely:
\vspace{-2mm}\begin{equation}
  \mathcal{L}_{\mathrm{QAM}} = -\frac{1}{N_{\mathrm{sym}}}
  \sum_{j=1}^{N_{\mathrm{sym}}} I_j
  \log\!\left(
    \frac{\exp(-\|\hat{y}_j - x_{y_j}\|^2 / N_0)}
         {\sum_{i=1}^{M}\exp(-\|\hat{y}_j -
         x_i\|^2 / N_0)}
  \right),
  \label{eq:qam_loss}
\vspace{-2mm}\end{equation}
where $\hat{y}_j$ is the received noisy symbol
for the $j$-th transmission, $x_{y_j}$ is the
 constellation point for true concept
index $y_j$, $N_0$ is the noise variance, and $I_j$ is the
normalized SCI score of the $j$-th transmitted
symbol.
 By scaling the cross-entropy by $I_j$,
the optimizer assigns disproportionately large gradient
penalties to decoding errors on high-importance symbols,
driving their constellation points toward regions of
maximum physical separation from co-occurring neighbors.
Training SNR is
randomized across $[\mathrm{SNR}_{\min},
\mathrm{SNR}_{\max}]$ at every step to ensure robustness
across diverse channel conditions.

\vspace{-1mm}\subsubsection{Evolution of the Semantic Aware Constellation}
To empirically validate this geometric adaptation, we visualize the optimization trajectory of a 256-QAM constellation for MNIST data  in Fig.~\ref{fig:256_QAM_semantic_constil_evolution}.
\begin{figure}[t]
    \centering
      \includegraphics[width=1\linewidth]{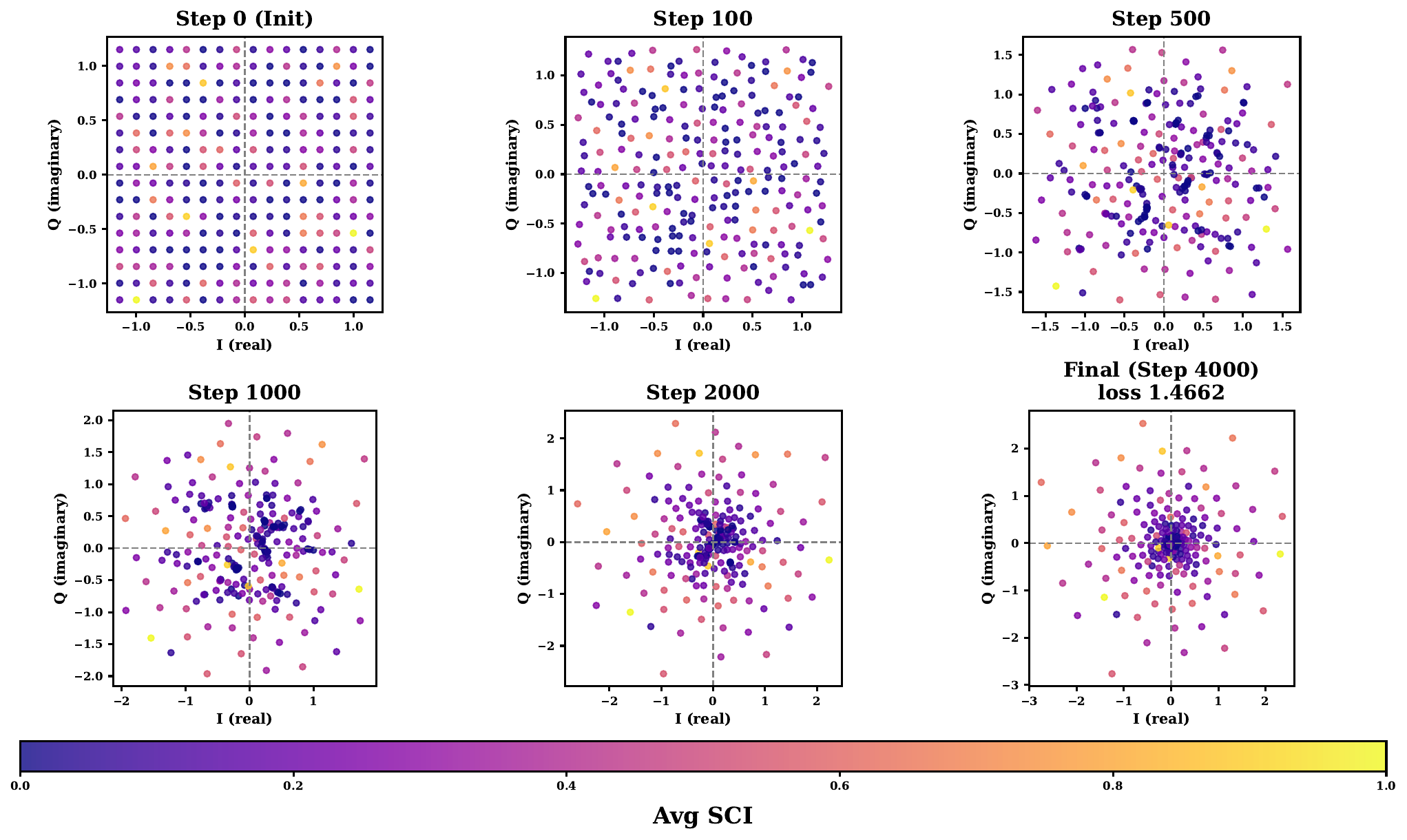}
   \vspace{-7mm} \caption{\small Evolution of semantic 256-QAM constellation. }
    \label{fig:256_QAM_semantic_constil_evolution}
 \vspace{-7mm}\end{figure}
The constellation initializes as a standard
rectangular grid and evolves under the
SCI-weighted loss into a task-aware geometry.
Symbols carrying low-importance concepts
naturally cluster in the dense interior of
the I/Q plane, sacrificing their decodability
to recover geometric space within the power
constraint.
Consequently, semantically critical symbols
are pushed toward the outer perimeter,
maximizing their physical separation and
directly validating the spatial isolation
strategy proved in Theorem~\ref{thm:main}. 
As eviden from Fig.~\ref{fig:256_QAM_semantic_constil_evolution}, only a distinct subset of physical symbols carries high semantic weight for any given class. This sparse physical importance matrix acts as the driving mathematical force behind the geometric evolution observed in Fig.~\ref{fig:256_QAM_semantic_constil_evolution}, pushing these specific high-value symbols toward the noise-resilient regions of the constellation.

\vspace{-2mm}\subsection{Deployment Considerations}\vspace{-2mm}
\vspace{-0mm}The proposed system follows a strict offline
training paradigm. The S-VQ-VAE, SCI, semantic
constellation, and DRL agent are trained jointly
on a central server prior to deployment. At
deployment, the lightweight encoder and SCI MLP
are loaded once onto the sensor node, while the
learned constellation is distributed as a static
lookup table of $M$ complex I/Q coordinates,
incurring negligible sharing overhead. The DRL
agent handles channel and SNR fluctuations
dynamically at inference time, so retraining is
not required for channel variation. Retraining
is only necessary if the source data distribution
shifts fundamentally, as the learned semantic
vocabulary of the VQ-VAE codebook would no longer
align with the new task context. Architectural details and
training hyperparameters are listed in Table~\ref{tab:hyper}.

\vspace{-1mm}\section{Optimality Analysis of Semantic Constellation Design}\vspace{-1mm}

\subsection{Semantic Symbol Vulnerability Metrics}\label{sec:ssv_formulation}\vspace{-2mm}
Evaluating the robustness of semantically critical payloads at the PHY requires moving beyond traditional bit-level error metrics. A decoding error occurs when channel noise displaces a transmitted symbol across a Voronoi boundary into a neighboring region. However, such errors are not equally consequential: errors on symbols encoding task-critical concepts can disrupt downstream inference, while others may have negligible impact. Standard BER treats all errors uniformly, making it a poor metric for semantic robustness. We therefore introduce a metric that weights physical vulnerability by semantic importance and co-occurrence structure.
\vspace{-2mm}\begin{definition}
The \emph{SCI-weighted SSV}
($\mathcal{S}_w$) of constellation $\mathcal{X}$
quantifies the expected physical vulnerability of
semantically critical symbols to decoding errors,
weighting each symbol's proximity to its
co-occurring neighbors by its average SCI
 score $\bar{I}_i$ and joint
co-occurrence probability $P(i,j)$:
\vspace{-4mm}\begin{equation}
\mathcal{S}_w
\;=\; \frac{1}{M}\sum\limits_{i\in \mathcal{X}}\bar{I}_i\underbrace{\left[M^2 \sum_{j \neq i}
 P(i,j)\,
 \exp\!\bigl(-\|x_i - x_j\|^2\bigr)\right]}_{\mathcal{S}_i}.
\label{eq:ssv}
\vspace{-2mm}\end{equation}
Here, the $M^2$ scaling ensures fair comparison
across modulation orders as  joint probabilities
shrink with increasing $M$.
\vspace{-6mm}\end{definition}
Under complex AWGN, the pairwise error probability between
symbols $i$ and $j$ satisfies:
$
\Pr[\hat{i}=j \mid x_i\;\text{sent}]
\;\leq\;
\frac{1}{2}\exp\!\left(
 -\frac{\|x_i - x_j\|^2}{4\sigma^2}
\right),
$
obtained by projecting the complex noise onto the
direction $x_j - x_i$ and applying the standard
Q-function bound. The exponential decay kernel
$\exp(-\|x_i-x_j\|^2)$ used in the $\mathcal{S}_w$ metric of \eqref{eq:ssv}
and the pairwise error bound are both strictly
decreasing functions of $\|x_i-x_j\|^2$ with identical
gradient directions.
Their gradients with respect to $x_i$ 
can be shown to be proportional to $(x_i - x_j)$
with strictly negative scalar prefactors, so they
point in identical directions for all $x_i \neq x_j$
and at any fixed SNR.  Since this directional
equivalence holds for each symbol pair $(i,j)$
independently, it extends to the
$P(i,j)$-weighted sum in
$\mathcal{S}_i$: minimizing $\mathcal{S}_w$ induces
the same optimal symbol placement as minimizing
the $P$-weighted pairwise error
probability at any fixed SNR, making $\mathcal{S}_w$
a SNR-agnostic proxy for PHY semantic vulnerability.
Crucially, symbols in $\mathcal{N}_i^c =
\{j: P(i,j)=0\}$ contribute zero to
$\mathcal{S}_i$ regardless of physical distance, capturing
the \emph{probabilistic isolation} effect: channel
confusions between mutually exclusive symbols cause no
semantic degradation. 
%
Finally, we define $\delta_i=\bar{I}_i - \frac{1}{M}\sum_{i=1}^{M}\bar{I}_i$ and the \emph{SCI score concentration} as
$
  \delta = \max_i\delta_i.$ Let
$\mathcal{S}_{\mathrm{top}} = \{i : \delta_i>0\}$ denote the set of symbols with above-average
SCI scores. This means that a symbol belongs to
$\mathcal{S}_{\mathrm{top}}$ if and only if it
contributes positively to $\delta$. Further, the \emph{SPP} $\mathcal{S}_p$ measures the fraction
of these symbols whose individual vulnerability is
strictly below the global mean vulnerability
$\mu_{\mathcal{S}} = \frac{1}{M}\sum_i
\mathcal{S}_i$:
\vspace{-2mm}\begin{equation}
 \mathcal{S}_p
 \;=\;
\frac{1}{|\mathcal{S}_{\mathrm{top}}|}
\sum_{i \in \mathcal{S}_{\mathrm{top}}}
\mathbf{1}\!\left(\mathcal{S}_i
 < \mu_{\mathcal{S}}\right).
\label{eq:sppr}
\vspace{-2mm}\end{equation}
Operationally, it represents the empirical probability that a semantically critical symbol, if chosen uniformly at random, is shielded better than the constellation average. Further, we formalize the \emph{co-occurrence asymmetry} as $
  \gamma = \max_{i,j} P(i,j) - \min_{i,j} P(i,j) \geq 0.$
When $\delta=0$ all concepts have equal SCI scores; when $\gamma=0$ all concept pairs co-occur with equal probability.

\vspace{-2mm}\begin{theorem}
\label{thm:main}
Let $\mathcal{X}_{\mathrm{QAM}}$ be a standard Gray-coded $M$-QAM constellation with average power $P$, and let $\mathcal{X}^*$ be the $\mathcal{S}_w$-minimizing constellation over all configurations in $\mathbb{C}^M$ subject to $\frac{1}{M}\sum_i|x_i|^2 \leq P$. If $\delta > 0$ and $\gamma > 0$, then $\mathcal{S}_w(\mathcal{X}^*) < \mathcal{S}_w(\mathcal{X}_{\mathrm{QAM}}),$
with protection gap $ \Delta_w = \mathcal{S}_w(\mathcal{X}_{\mathrm{QAM}})-\mathcal{S}_w(\mathcal{X}^*)$ lower-bounded by
\vspace{-2mm}\begin{equation}
  \Delta_w \geq \frac{\delta\cdot\gamma\cdot M}{1 + \zeta^*/w_{\max}} \left[ \exp(-d_{\min}^2) - \exp(-d_{\max}^2) \right],
  \label{eq:gap_bound}
\end{equation}
where $d_{\min}$ is the minimum inter-symbol distance of $\mathcal{X}_{\mathrm{QAM}}$, $d_{\max}$ is the maximum feasible inter-symbol distance under power $P$, $w_{\max} = \max_{i,j}(\bar{I}_i+\bar{I}_j) P(i,j)$, and $\zeta^*$ is the Lagrange multiplier of the optimal solution.
\end{theorem}

\vspace{-2mm}\begin{IEEEproof}[Proof Sketch]
The $\mathcal{S}_w$ minimisation over $\mathbb{C}^M$ with power constraint admits the Lagrangian $\mathcal{L} = \mathcal{S}_w(\mathcal{X}) + \zeta(\frac{1}{M}\sum_i|x_i|^2 - P)$. Using $P(i,j)=P(j,i)$, the KKT stationarity condition at $\mathcal{X}^*$ is
\vspace{-1mm}\begin{equation}
  \frac{2\zeta^*}{M} x_i^* = 2\sum_{j\in\mathcal{N}_i} w_{ij}^* (x_i^* - x_j^*),
  \label{eq:kkt}
\vspace{-2mm}\end{equation}
where $w_{ij}^* = (\bar{I}_i+\bar{I}_j) P(i,j)\exp(-\|x_i^*-x_j^*\|^2)$ and the sum runs only over $\mathcal{N}_i$ since $w_{ij}^*=0$ for $j\in\mathcal{N}_i^c$. The entire power budget for symbol $i$ is therefore directed toward separating it from its semantically
coupled neighbors, with zero budget wasted on non-co-occurring neighbors. 
For~\eqref{eq:kkt} to hold at
$\mathcal{X}_{\mathrm{QAM}}$, the weights $w_{ij}$
must be symmetric under all symmetry operations of
the rectangular grid, requiring uniform $w_{ij}$
across all nearest-neighbor pairs.
Since $\delta > 0$, the importance values
$\bar{I}_i$ are non-uniform, and since $\gamma > 0$,
the co-occurrence probabilities $P(i,j)$ are
non-uniform. Since Gray coding assigns symbol
positions independently of $\bar{I}_i$ and $P(i,j)$,
the products $(\bar{I}_i + \bar{I}_j)P(i,j)$ are
non-uniform across nearest-neighbor pairs, violating
the uniformity condition. Therefore
$\nabla_{x_{i^\dagger}}\mathcal{S}_w
|_{\mathcal{X}_{\mathrm{QAM}}} \neq \mathbf{0}$
for at least one symbol $i^\dagger$, and
$\mathcal{X}_{\mathrm{QAM}}$ is not a stationary
point of $\mathcal{L}$.  Since $\mathcal{X}_{\mathrm{QAM}}$ is not
stationary, there exists a perturbation
$\mathcal{X}_\epsilon$ feasible under the power
constraint such that
$\mathcal{S}_w(\mathcal{X}_\epsilon) <
\mathcal{S}_w(\mathcal{X}_{\mathrm{QAM}})$.
Since $\mathcal{X}^*$ globally minimizes
$\mathcal{S}_w$ over the feasible set:
$ \mathcal{S}_w(\mathcal{X}^*)
  \leq \mathcal{S}_w(\mathcal{X}_\epsilon)
  < \mathcal{S}_w(\mathcal{X}_{\mathrm{QAM}}).$ Since $\delta > 0$, there exists
$i^\dagger = \arg\max_i \bar{I}_i$ with
$\bar{I}_{i^\dagger} \geq \mu_{\bar{I}} + \delta$.
Since $\gamma > 0$, there exists a pair
$(i^\dagger, j^\dagger)$ with $P(i^\dagger,
j^\dagger) \geq \gamma/M^2$ after $M^2$
normalisation. Since Gray coding places symbols
independently of co-occurrence structure, this
pair is separated by at most $d_{\min}$ on the
uniform grid. The monotone decay of the
exponential kernel gives the contribution of
this pair to $\mathcal{S}_w(\mathcal{X}_
{\mathrm{QAM}})$ as at least
$(\mu_{\bar{I}} + \delta)\cdot\gamma\cdot
\exp(-d_{\min}^2)$. Summing over all $M$
symbols with the $M$ prefactor:
\begin{equation*}
  \mathcal{S}_w(\mathcal{X}_{\mathrm{QAM}})
  \geq \delta\cdot\gamma\cdot M\cdot
  \exp(-d_{\min}^2).
\end{equation*}
At $\mathcal{X}^*$, the power constraint bounds
$|x_i^*|^2 \leq MP$, so the maximum feasible
inter-symbol distance is $d_{\max} = 2\sqrt{MP}$.
The KKT force balance in~\eqref{eq:kkt} shows
that the effective separation scales as
$w_{\max}/(w_{\max} + \zeta^*/M)$, yielding the
factor $1/(1 + \zeta^*/w_{\max})$. At maximum
separation $d_{\max}$:
\begin{equation*}
  \mathcal{S}_w(\mathcal{X}^*)
  \leq \frac{\delta\cdot\gamma\cdot M}
       {1 + \zeta^*/w_{\max}}
  \cdot\exp(-d_{\max}^2).
\end{equation*}
 
\emph{Subtracting and using
$1/(1+\zeta^*/w_{\max}) \leq 1$ and
$d_{\min} < d_{\max}$:}
\begin{align*}
  \Delta_w
  &\geq \delta\cdot\gamma\cdot M\cdot
    \exp(-d_{\min}^2)
    - \frac{\delta\cdot\gamma\cdot M}
           {1+\zeta^*/w_{\max}}
      \exp(-d_{\max}^2) \\
  &\geq \frac{\delta\cdot\gamma\cdot M}
             {1+\zeta^*/w_{\max}}
    \!\left[\exp(-d_{\min}^2)
           - \exp(-d_{\max}^2)\right],
\end{align*}
establishing~\eqref{eq:gap_bound}.
\end{IEEEproof}
\vspace{-2mm}\begin{corollary}
\label{cor:ber}
Under the conditions of Theorem \ref{thm:main}, $\mathcal{X}^*$ achieves strictly lower $\mathcal{S}_w$ and strictly lower semantic error than $\mathcal{X}_{\mathrm{QAM}}$, while exhibiting strictly higher average BER.
\end{corollary}

\vspace{-2mm}\begin{IEEEproof}
The $\mathcal{S}_w$ reduction follows from Theorem \ref{thm:main}. The average BER increases because the descent direction that reduces $\mathcal{S}_w$ crowds low-SCI symbols into high-density interior regions, increasing their individual error probability. Since these errors fall on semantically negligible symbols, the semantic error decreases simultaneously.
\end{IEEEproof}

\vspace{-2mm}\section{Simulation Results and Analysis}\vspace{-1mm}
We evaluate the proposed system on MNIST, Fashion-MNIST, and the Free Spoken Digit Dataset (FSDD), spanning basic image classification, complex visual feature extraction, and audio processing to demonstrate its cross-domain multimodal capability.
All neural networks are implemented in TensorFlow and trained on an NVIDIA DGX Spark server, with the PHY simulated using the GPU-accelerated Sionna PHY library~\cite{sionna}. To facilitate reproducibility, the complete
source code are publicly
available at~\footnote{\url{https://github.com/THE-TRAIN-LAB/Semantic-QAM}}. The wireless channel is modeled as AWGN, evaluated over $[-10, 20]$\,dB SNR. Modulation orders span $M \in \{4, 16, 64, 256, 1024\}$, with the codebook size constrained to $|\mathcal{C}| = M$ in each configuration. All numeric hyperparameters are listed in Table~\ref{tab:hyper}. The \emph{standard $M$-QAM} baseline retains the SC pipeline but uses a fixed rectangular grid instead of a learned constellation. 
The composite semantic quality score is defined as: $\mathcal{Q}_{\mathrm{sem}}
  = 0.6\,Q_{\mathrm{task}}
  + 0.25\,P_{c}
  + 0.15\exp(-D_{\mathrm{KL}}({p} \| \hat{p})),$
where $Q_{\mathrm{task}}$ is strict classification accuracy, $P_{c} = 1 -
\mathbb{E}[|\max({p}) -
\max(\hat{p})|]$ measures
how faithfully the peak classifier confidence is
preserved, where ${p}$ and
$\hat{p}$ are the softmax
output distributions of 
$\mathcal{T}$ evaluated on the original and
reconstructed images respectively. $ \exp(-D_{\mathrm{KL}}({p} \| \hat{p}))$ measures distribution similarity using KL divergence $D_{\mathrm{KL}}$. The weighting prioritizes strict task accuracy over classifier output distribution fidelity.

\vspace{-2mm}\subsection{Semantic Quality and the BER Paradox}\vspace{-2mm}
 \begin{figure}[t]
    \centering
    \includegraphics[width=0.7\linewidth]{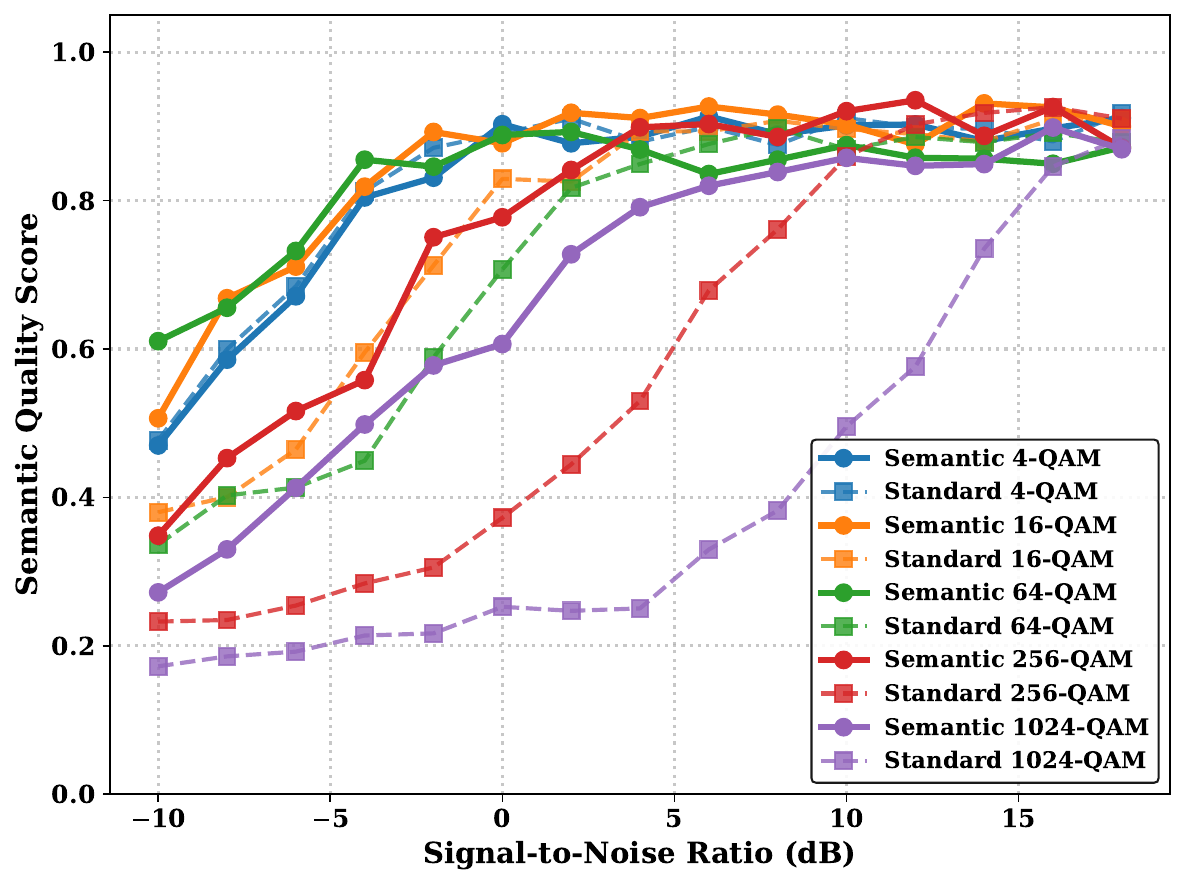}
    \vspace{-3mm}\caption{\small Semantic Quality across distinct M-QAMs for MNIST.} 
    \label{fig:sem_q_vs_snr}
\vspace{-5mm}\end{figure}

Fig.~\ref{fig:sem_q_vs_snr} compares $\mathcal{Q}_{\mathrm{sem}}$ across all modulation orders for the MNIST dataset. Semantic $M$-QAM consistently
outperforms standard $M$-QAM across the full SNR
range, with average gains of approximately 40\%
at low SNR ($-10$ to $0$\,dB) and 15\% at high
SNR ($5$ to $15$\,dB), with the gap widening
at higher modulation orders where Semantic
1024-QAM achieves $\mathcal{Q}_{\mathrm{sem}}
\approx 0.60$ versus $0.25$ for its standard
counterpart at $0$\,dB. This demonstrates that the learned constellation maintains task accuracy even at low-SNR regime by ensuring high SCI concepts are well separated from others. The narrower gain at $M=4$ is a direct consequence of the codebook constraint: with only 4 concepts, the SCI cannot finely decouple task-critical features from background noise. 


Fig.~\ref{fig:ber_vs_snr} shows that semantic $M$-QAM exhibits strictly \emph{higher} average BER at higher SNR than Standard $M$-QAM across all modulation orders. This is because the learned mapper deliberately crowds low-SCI symbols into the centre of the I/Q plane, sacrificing their decodability to maximize physical separation for task-critical concepts. Since these bit-level errors fall entirely on semantically negligible symbols, they have near-zero impact on $\mathcal{Q}_{\mathrm{sem}}$. This empirically validates Corollary~\ref{cor:ber}. 
\begin{figure}[t]
    \centering
    
    \begin{minipage}{0.48\linewidth}
        \centering
        \includegraphics[width=\linewidth]{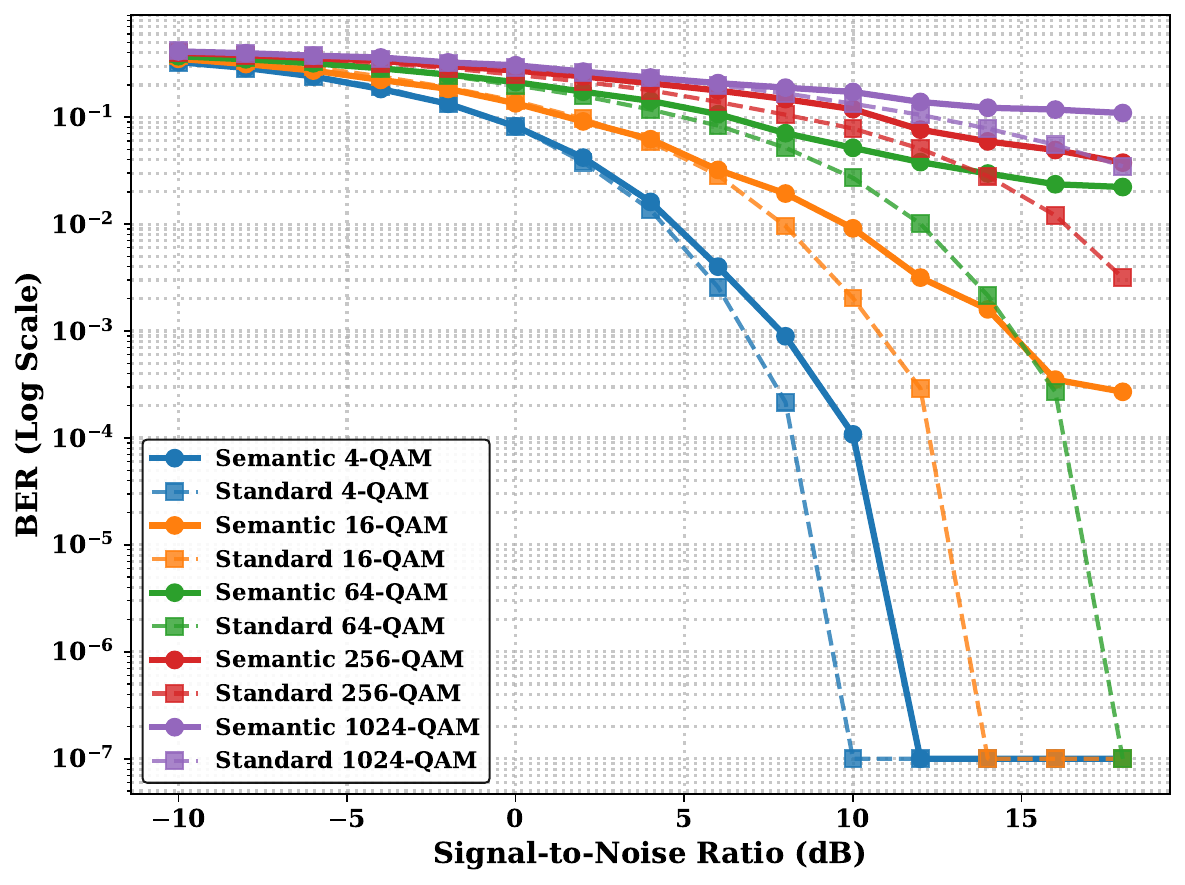}
        \vspace{-7mm} 
        \caption{\scriptsize Bit Error Rate vs. SNR (dB).}
        \label{fig:ber_vs_snr}
    \end{minipage}%
    \hfill
    \begin{minipage}{0.48\linewidth}
        \centering
        \includegraphics[width=\linewidth]{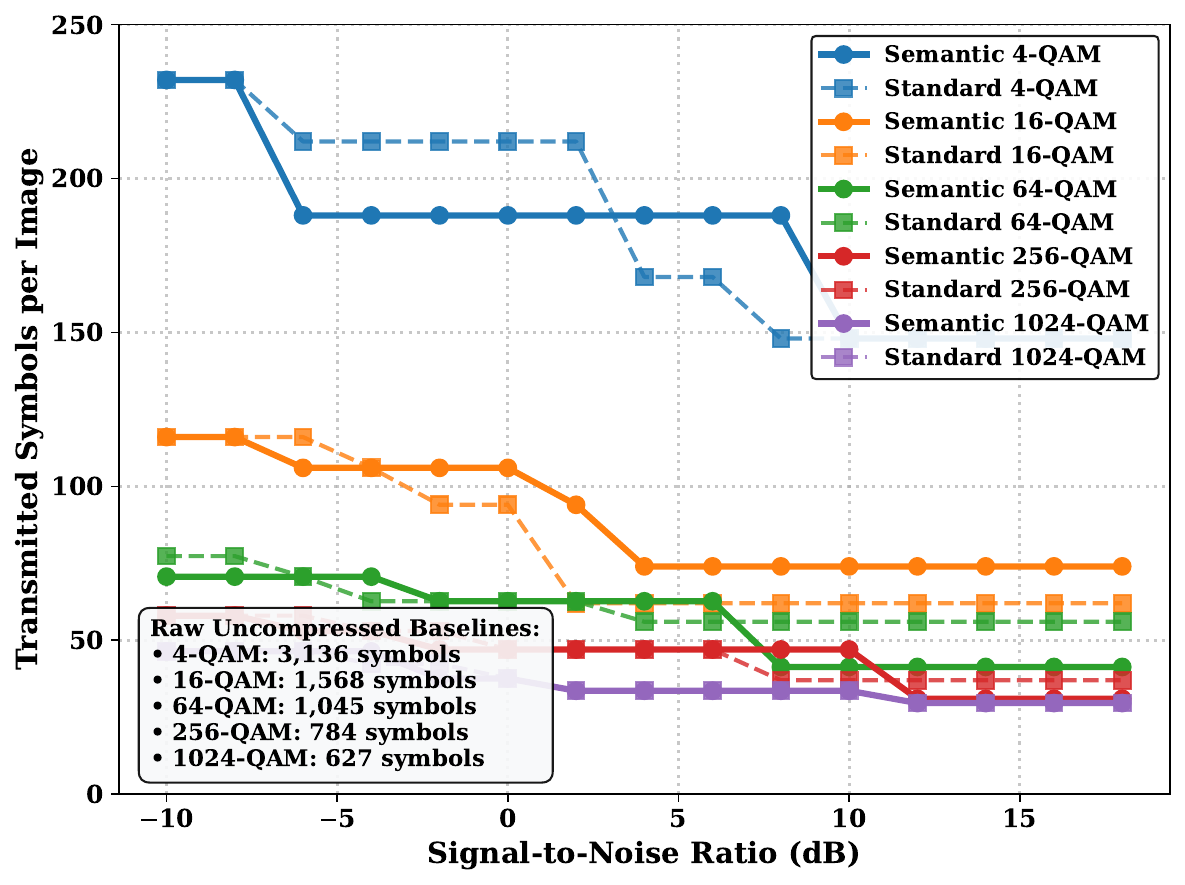}
        \vspace{-7mm} 
        \caption{\scriptsize Symbols transmitted vs. SNR.}
        \label{fig:sym_tx_per_im_vs_snr}
    \end{minipage}
    
    \vspace{-7mm} 
\end{figure}
\vspace{-0mm}\subsection{Adaptive Compression and Latency}\vspace{-2mm}

Fig.~\ref{fig:sym_tx_per_im_vs_snr} shows that the DRL agent scales $K$ inversely with SNR, using larger $K$ for redundancy at low SNR and aggressive Top-$K$ compression at high SNR. Even at $-10$ dB, semantic $1024$-QAM transmits fewer than $40$ symbols per image versus $627$ in the baseline ($>15\times$ reduction), exceeding $20\times$ at high SNR. This demonstrates that joint semantic compression and PHY protection improves $\mathcal{Q}_{\mathrm{sem}}$ while reducing symbol count across all SNRs.

\vspace{-2mm}\subsection{Cross-Domain Generalization}\vspace{-2mm}
The system’s cross-domain applicability is validated on the Fashion-MNIST visual dataset \cite{xiao2017fashion} and the audio-based FSDD dataset, where semantics are extracted from audio spectrograms. Fig.~\ref{fig:cross_domain_results} shows that the semantic constellation consistently outperforms the standard baseline across all modulation orders and SNR regimes. The adaptive compression and BER paradox behaviors remain consistent across modalities, indicating that the semantic constellation architecture is robust and dataset-agnostic.
\begin{figure}[t]
    \centering
    \subfloat[]{\includegraphics[width=0.48\linewidth]{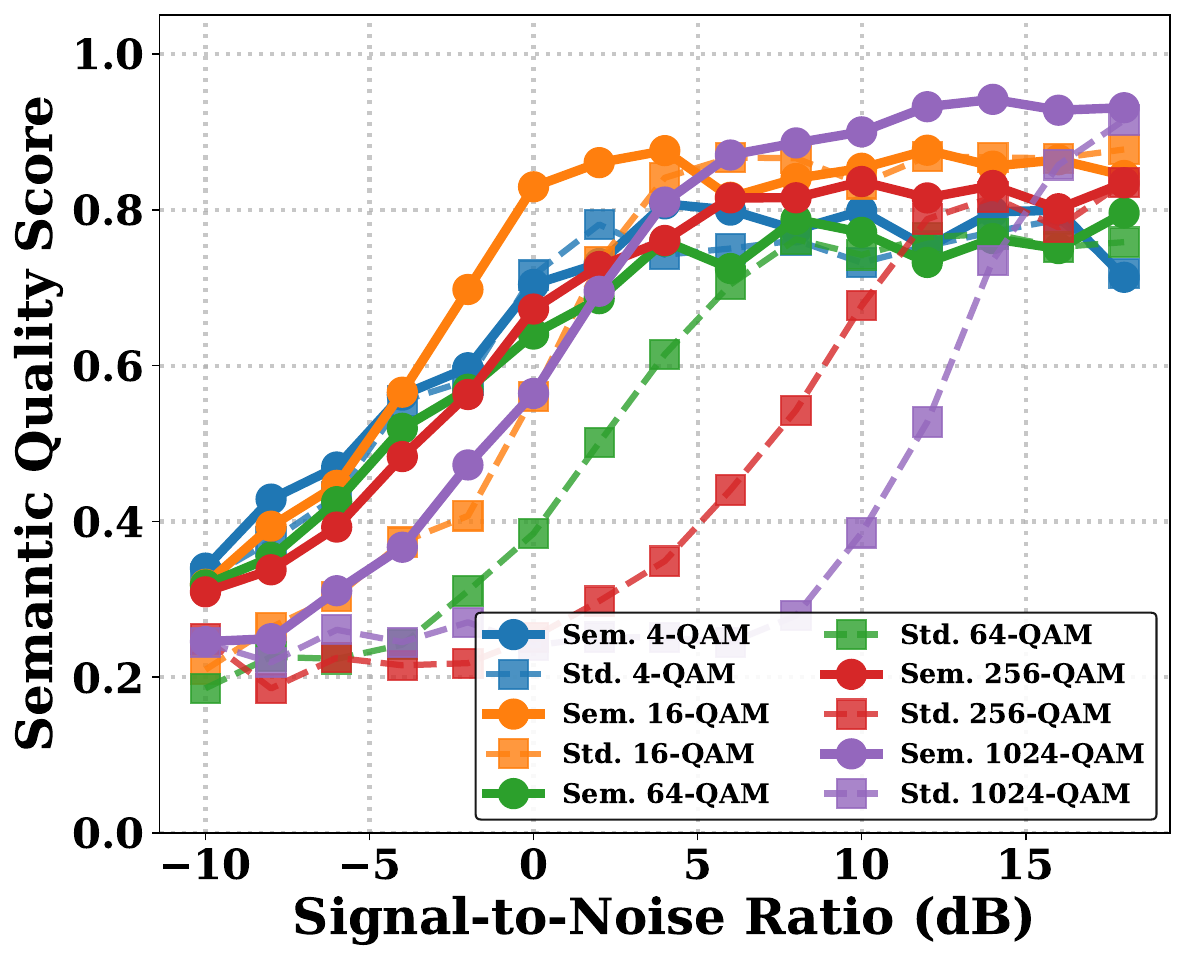}\label{fig:fsdd_sem}}
    \hfill
    \subfloat[]{\includegraphics[width=0.48\linewidth]{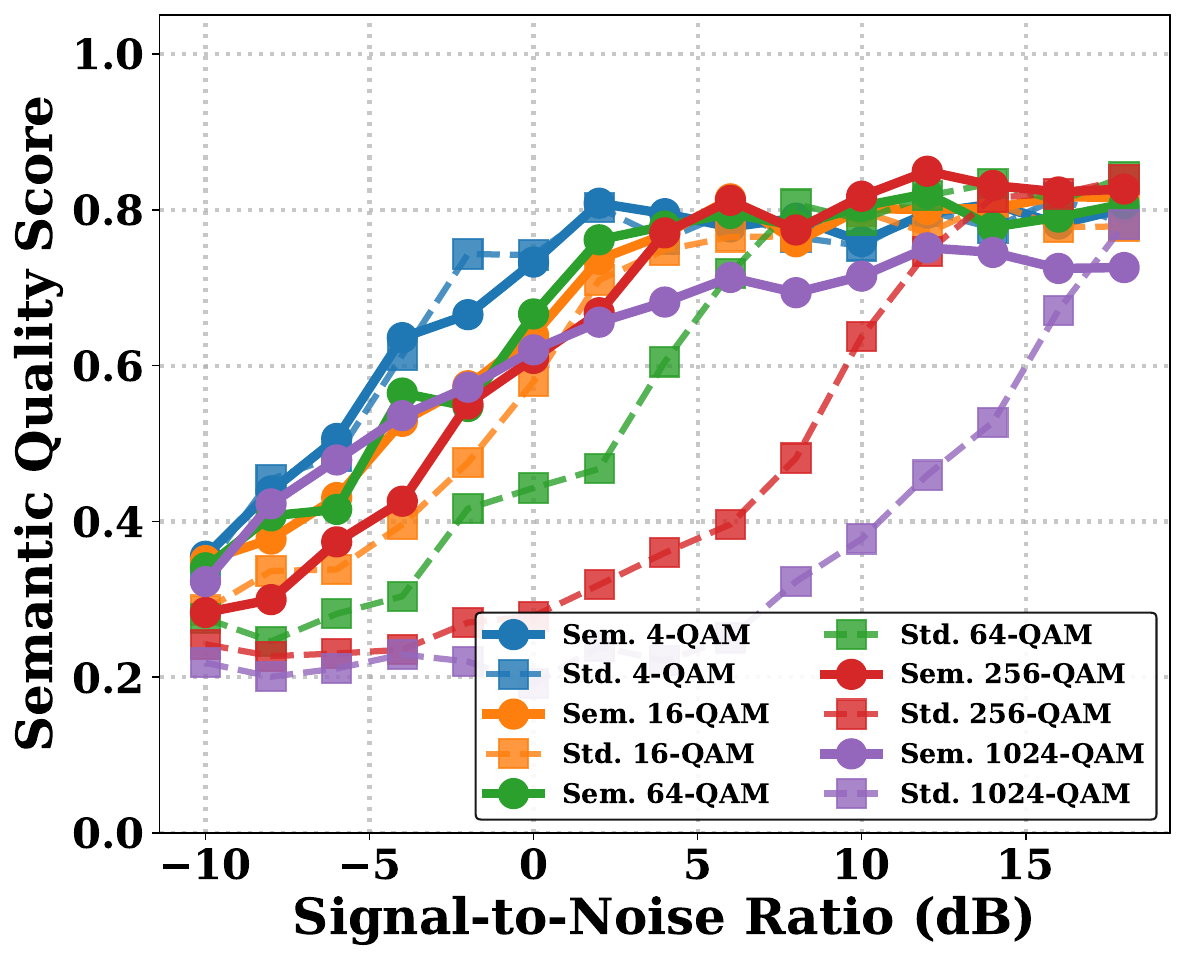}\label{fig:fmnist_sem}}
    \vspace{-3mm}\caption{\small  Semantic Quality vs.\ SNR (dB) for FSDD audio (a) and Fashion-MNIST visual (b) datasets.}
    \label{fig:cross_domain_results}
\vspace{-7mm}\end{figure}

\vspace{-2mm}\subsection{Semantic Symbol Vulnerability Analysis}\vspace{-1mm}
\begin{figure}[t]
  \centering
  \subfloat[]{\includegraphics[width=0.48\linewidth]{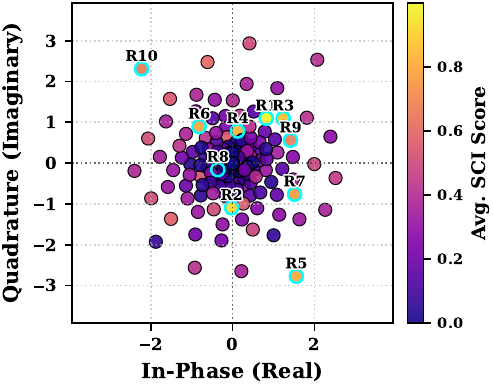}}
  \hfill
  \subfloat[]{\includegraphics[width=0.48\linewidth]{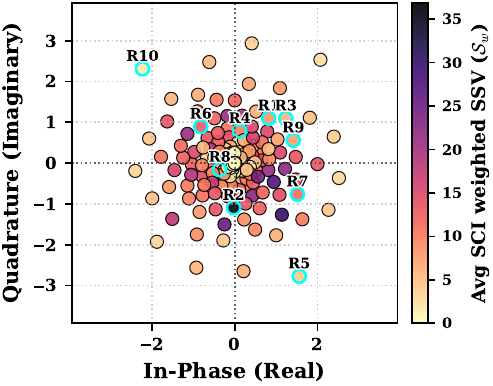}}
  \vspace{-3mm}\caption{\small Learned 256-QAM constellation with respect to (a) average SCI and (b) $\mathcal{S}_w$. Cyan rings highlight the top-10 most critical concepts (R1--R10).}
  \label{fig:constellation}
\vspace{-8mm}\end{figure}
 
In Fig.~\ref{fig:constellation}, the left panel colours each learned 256-QAM symbol by its average SCI. The right panel maps the $\mathcal{S}_w$  onto the same constellation. Two  protection strategies emerge. High-SCI symbols with large $|\mathcal{N}_i|$ are pushed to the outer perimeter, maximizing physical distance from their co-occurring neighbors (\emph{spatial isolation}). Conversely, several high-SCI symbols reside safely in the dense interior, where their vulnerability is negligible because $P(i,j) \approx 0$ for all physical neighbors $j$ (\emph{probabilistic isolation}). This behavior is precisely the mechanism characterized by the stationarity condition in~\eqref{eq:kkt} and cannot emerge from any importance-blind assignment, including Gray-coded QAM. Fig.~\ref{fig:cross_dataset_evaluation} confirms
that the semantic constellation consistently
suppresses $\mathcal{S}_w$ across all modulation
orders, with the gap growing with $M$ as predicted
by~\eqref{eq:gap_bound}, while maintaining near
100\% $\mathcal{S}_p$ versus roughly 50\% for
standard constellations at $M = 1024$.

\vspace{-2mm}\begin{figure}[t]
    \centering
    \subfloat[\vspace{-1mm}$\mathcal{S}_w$ vs.\ $M$-QAM]{\includegraphics[width=0.48\linewidth]{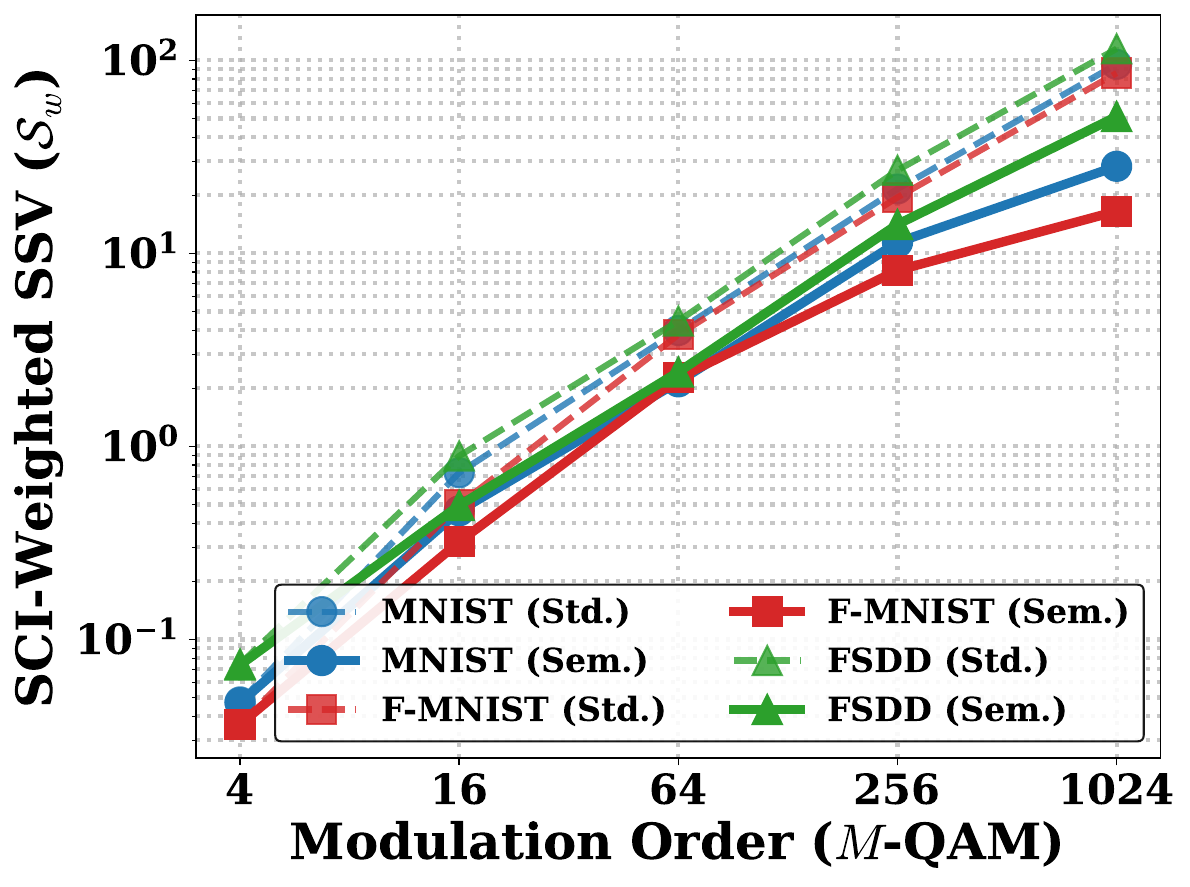}}
    \hfill
    \subfloat[\vspace{-1mm}$\mathcal{S}_p$ vs.\ $M$-QAM]{\includegraphics[width=0.48\linewidth]{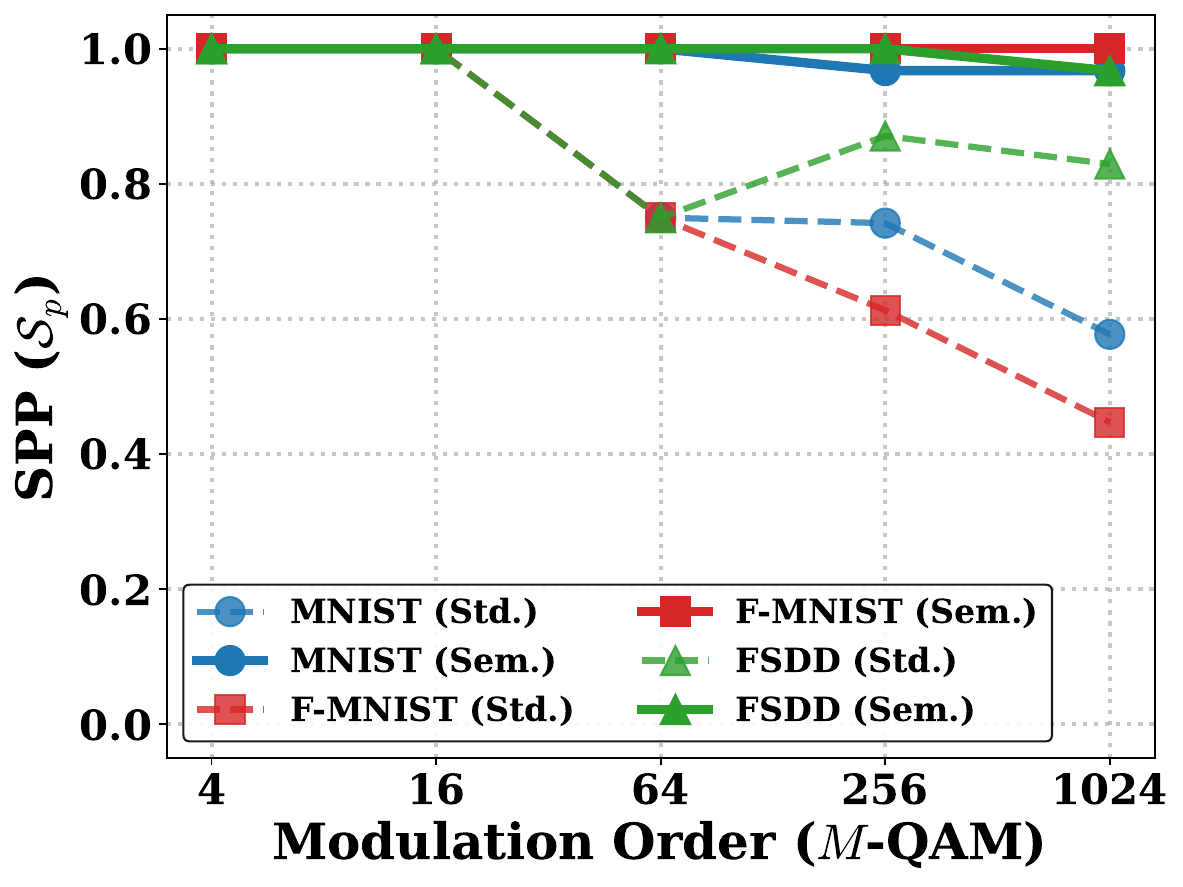}}
    \vspace{-2mm}\caption{SSV and SPP analysis across modulation orders.}
    \label{fig:cross_dataset_evaluation}
\vspace{-7mm}\end{figure}

\vspace{-1mm}\section{Conclusion}\vspace{-1mm}
In this paper, we have introduced a novel semantic QAM architecture that prioritizes the physical-layer protection of task-critical features. By leveraging a DRL-based rate controller and an SCI-weighted loss function, the proposed system natively isolates highly semantic concepts from channel noise. Empirical evaluations across visual and acoustic datasets confirm that our approach maintains near $100\%$ protection for critical symbols and achieves massive compression gains over standard baselines, even in severely degraded SNR regimes. This framework offers a robust, scalable foundation for next-generation AI-native 6G networks.

\vspace{-3mm}\bibliographystyle{IEEEbib}
\def\baselinestretch{0.84}
\bibliography{refs} 
\vspace{-1mm}
\end{document}